\documentclass[times]{article}
\pdfoutput=1 
\usepackage{proceed2e}

\usepackage{amsmath}
\usepackage{amsfonts}
\usepackage{amssymb}
\usepackage{amsthm}
\usepackage{graphicx}
\usepackage{float}
\usepackage{enumerate}

\newtheorem{thm}{Theorem}
\newtheorem*{unthm}{Theorem}
\newtheorem{lem}{Lemma}

\DeclareMathOperator*{\sign}{sign}

\DeclareMathOperator*{\argmin}{argmin}
\floatstyle{ruled}
\newfloat{algorithm}{tbp}{loa}
\floatname{algorithm}{Algorithm}

\title{Online Importance Weight Aware Updates}

\author{{\bf Nikos Karampatziakis}\thanks{~Research done while author was at Yahoo! Research}\\
Department of Computer Science\\
Cornell University\\
\texttt{nk@cs.cornell.edu}
\And 
{\bf John Langford}\\
Yahoo! Research\\
\texttt{jl@yahoo-inc.com}
}

\begin{document}

\maketitle 
 
\begin{abstract} 
  An importance weight quantifies the relative importance of one example
  over another, coming up in applications of boosting, asymmetric
  classification costs, reductions, and active learning.  The standard
  approach for dealing with importance weights in gradient descent is
  via multiplication of the gradient.  We first demonstrate the problems
  of this approach when importance weights are large, and argue in favor
  of more sophisticated ways for dealing with them.  We then develop an
  approach which enjoys an invariance property: that updating twice with
  importance weight $h$ is equivalent to updating once with importance
  weight $2h$. For many important losses this has a closed form update
  which satisfies standard regret guarantees when all examples have
  $h=1$.  We also briefly discuss two other reasonable approaches for
  handling large importance weights.  Empirically, these approaches
  yield substantially superior prediction with similar computational
  performance while reducing the sensitivity of the algorithm to the
  exact setting of the learning rate.  We apply these to online active
  learning yielding an extraordinarily fast active learning algorithm
  that works even in the presence of adversarial noise.
\end{abstract}

\section{INTRODUCTION}

Importance weights appear in boosting algorithms~\cite{freund1995decision} which
assign a weight to each example depending on how well this point has
been classified in previous iterations, covariate shift
algorithms~\cite{huang2007correcting} which assign a weight to a training
example according to how close to the test distribution the example
is, and active learning algorithms~\cite{beygelzimer2009importance,beygelzimer2010agnostic} where an adaptive
rejection sampling scheme is applied to each example and each retained
example gets an importance equal to the inverse probability of being
retained. Importance weights have become a de-facto language for specifying the
relative importance of prediction amongst examples.

When not concerned by computational constraints, importance weights
can be dealt with using either black box
techniques~\cite{Schapire1990,zadrozny2003cost} or direct modification of
existing algorithms such that an existing example with importance
weight $h$ is treated as $h$ examples.  However, when computational
constraints are significant online gradient descent based algorithms
are preferred.  Here the standard approach of treating an example with
importance weight $h$ as $h$ examples is typically translated into
practice via multiplying the gradient by $h$. This is undesirable for
large $h$ because such an example can cause an update that's far 
beyond what's necessary to attain a small loss on it.

An important observation is that multiplying the gradient by $h$ is
typically \emph{not} equivalent to doing $h$ updates via gradient
descent because all loss functions of interest are nonlinear.  The
goal of this paper is resolving this translation failure by investigating
alternate updates that gracefully deal with importance weights, by
taking into account the curvature of the loss.  Among these updates 
we mainly focus on a novel set of updates that satisfies an additional 
\emph{invariance} property: for all importance weights of $h$,
the update is equivalent to two updates with importance
weight $h/2$. We call these updates \emph{importance invariant}. 

Even though the importance invariant updates will be defined via 
an ordinary differential equation (ODE), we were surprised to 
find that they are closed-form for all common loss functions.  We were
also surprised to discover that the importance weight invariant update
substantially improves the learned predictor even when $h=1$, both in
terms of the quality of best predictor after a parameter search and in
terms of the robustness to parameter search, effectively reducing the
desirability of searching over some schedule of learning rates.  Upon
inspection, the reason for this is that an importance weight invariant
update smoothly interpolates between a very aggressive
projection~\cite{herbster2001learning,crammer2006online} algorithm and 
a less aggressive gradient multiplier decay algorithm.  All of these 
benefits come at near-zero computational cost.

Among the other algorithms we consider, implicit updates \cite{ig1,ig2} 
turn out to coincide with importance invariant ones for piecewise linear 
losses and provide qualitatively similar updates for other losses, 
implying that our derivation is an alternative way to motivate this 
style of algorithm. For most other loss functions implicit updates require a 
root-finding algorithm. 

Finally, another reasonable way to handle importance is
related to \cite{dekel2005smooth}, who analyze it for the logistic and 
exponential losses. Here, a second order Taylor 
expansion at the current prediction and in the direction of the update 
is used to approximate the loss. These updates coincide with implicit 
updates for squared loss (since the quadratic approximation is exact)
and are not applicable to piecewise linear losses. We won't discuss
these updates any further since they have only been analyzed for very
specific loss functions. 

In section \ref{setting} we define the problem and describe some
obvious but unsatisfactory approaches. Next, we propose the importance
invariant solution and present a general framework for deriving
importance invariant updates for many loss functions, in
section~\ref{solution}.  We subsequently discuss some important
properties of the proposed updates, such as safety, in
section~\ref{properties}. Section~\ref{imp2} briefly covers 
how implicit updates can handle importance weights. 
In section~\ref{exps} we empirically
demonstrate the merits of not linearizing the loss on problems with and without
importance weights. Section~\ref{concl} states our conclusions.

\section{PROBLEM SETTING 
} \label{setting}

We assume access to a training set of triplets 
$(x_t,y_t,h_t)$, $t=1,\ldots,T$ where $x_t \in \mathbb{R}^d$ 
is a vector of $d$ features, $h_t \in \mathbb{R}_+$ is an importance
weight, and $y_t \in \mathbb{R}$ is a label. We are also given a
loss function $\ell(p,y)$ where $p$ is the prediction of our 
model and $y$ is the actual label. Depending on the loss function,
$y$ may take values in a restricted set, such as $\{-1,+1\}$ 
or $\{0,1\}$. In this paper we focus on
linear models i.e. $p=w^\top x$ where $w \in \mathbb{R}^d$
is a vector of weights. Our goal is to find
\begin{equation}\label{eq:weighted-objective}
w = \argmin_{w} \sum_{t=1}^T h_t \ell(w^\top x_t,y_t),
\end{equation}
using online gradient descent. When examples do \emph{not} have 
importance weights the online gradient descent algorithm is 
shown in Algorithm~\ref{tab:sgd}. The notation 
$\nabla_w \ell(w_t^\top x_t,y_t)$ means the gradient of the loss
with respect to $w$ evaluated at the $t$-th prediction
and label.

\begin{algorithm}
\caption{\label{tab:sgd}Online Gradient Descent}
$w_{1} \leftarrow 0$\\
\textbf{for} $t=1$ to $T$ \textbf{do}\\
\mbox{~~~~}$w_{t+1}\leftarrow w_{t}-\eta_{t}\nabla_{w}\ell(w_{t}^{\top}x_{t},y_{t})$\\
\textbf{done}
\end{algorithm}

When examples have importance weights, we would like to adhere
to the following principle:
\emph{An example with importance weight $h$ should be treated as 
if it is a regular example that appears $h$ times in the dataset.} 
This is a statement of both mathematical and semantic correctness.
Mathematically, (\ref{eq:weighted-objective}) states exactly the same thing. 
Semantically an example of importance $h$ is just a convenient 
encoding of $h$ identical examples. 
For now we assume importance weights are integers 
and the learning rate sequence is constant $\eta_t = \eta$. 
These assumptions are only for ease of exposition and are 
lifted in section 3.

\subsection{Some Unsatisfactory Approaches}

A first approach would be to loop through the data
enough epochs, with epoch $i$ using only those examples
whose importance is greater than $i$.
While this is a valid approach,
it is very inefficient. Ideally each example should be 
presented once to the learner.

Another tempting approach is multiplying the update by the importance weight:
\[
w_{t+1} = w_t - h_t\eta \nabla_w \ell(w_t^\top x_t,y_t).
\]
However, this update rule does not respect the principle 
of the previous section. To see this, consider the case 
$h_t=2$. The above rule should be equivalent to 
\begin{eqnarray*}
v & = &  w_t - \eta \nabla_w \ell(w_t^\top x_t,y_t)\\
w_{t+1} & = &  v - \eta \nabla_w \ell(v^\top x_t,y_t)
\end{eqnarray*}
which is not true in general. Furthermore, the quality of this update
gets worse as the importance weight gets larger since the first order
approximation of the loss is invalid far away from its expansion
point.

Another approach with good computational characteristics is rejection
sampling accoding to $h/h_{\max}$.  However,
this approach generally decreases performance due to throwing out
samples.  Rejection sampling can be repaired by learning multiple
predictors based upon different rejection sampled
datasets~\cite{zadrozny2003cost}, but this of course increases computation
substantially.

\subsection{An Efficient Invariant Approach}

To achieve invariance and efficiency we will focus on
the cumulative effect of presenting an example $h$ times 
in a row. This scheme respects our correctness principle
and considers each example once. The only remaining 
question is whether the cumulative effect of $h$ 
presentations in a row can be computed faster than 
explicitly doing so. In the next section we explain
why this is possible and how to compute it. 

\section{A FRAMEWORK FOR DERIVING STEP SIZES} \label{solution}

Given a loss of the form $\ell(p, y)$ where $p$ is the prediction and 
assuming a linear model $p=w^\top x$ we have that 
$\nabla_w \ell = \frac{\partial \ell}{\partial p} x$.
Therefore all gradients of a given example point to the same direction 
and only differ in magnitude. Hence computing the cumulative effect of 
presenting the example $h$ times in a row amounts to computing a global
scaling for $x$ that aggregates the effects of all the gradients. We 
begin with a simple lemma that formalizes this in the case of 
integer importance weights.
\begin{lem} \label{lem:int}
Let $h \in \mathbb{N}$. Presenting example $(x,y)$ $h$ times in a row 
is equivalent to the update
\begin{equation}\label{eq:general}
 w_{t+1} = w_t - s(h) x 
\end{equation}
where the scaling factor $s(h)$ has this recursive form:
\begin{eqnarray}
 s(h+1) & = & s(h)+\eta \left.\frac{\partial \ell}{\partial p}\right|_{p=(w_t -s(h)x)^\top x} \label{eq:rec} \\
 s(0) & = & 0 \label{eq:base}
\end{eqnarray}
\end{lem}
\begin{proof}
By induction on $h$. The base case is obvious. Now the effect of 
presenting the example $(x,y)$ $h+1$ times can be computed by 
performing a gradient update on the vector $v$ that results 
from presenting the example $h$ times.
By the induction hypothesis this intermediate vector is
\[
 v = w_t -s(h) x
\]
and the gradient descent step is
\[
 w_{t+1} = v - \eta \nabla_w \ell(w^\top x,y)|_{w=v}.
\]
Expanding this using the induction hypothesis we get
\begin{eqnarray*}
 w_{t+1} & =  & w_t -s(h) x -\eta  \left.\frac{\partial \ell}{\partial p}\right|_{p=v^\top x} x\\
         & =  & w_t -\left(s(h)+\eta \left.\frac{\partial \ell}{\partial p}\right|_{p=(w_t -s(h)x)^\top x}\right)x
\end{eqnarray*}
\end{proof}

Given a loss function, one could try to find a closed form 
solution to the recurrence defined by (\ref{eq:rec}),(\ref{eq:base}).
For example for squared loss $\ell(p,y)=\frac{1}{2}(p-y)^2$
the recurrence is
\[
 s(h+1) =  s(h)+\eta((w_t -s(h)x)^\top x -y).
\] 
A simple inductive argument can then verify that 
\begin{equation}\label{eq:closed1}
s(h)=\frac{w_t^\top x -y}{x^\top x}(1-(1-\eta x^\top x)^h)
\end{equation}
Note that when $\eta x^\top x<1$, $s(h)$ asymptotes to the
quantity that would make $w_{t+1}^\top x =y$. This behavior
is more desirable than that of multiplying the gradient 
with the importance weight.

Notwithstanding the significance of such an update,
it is restricted to integer importance weights. Moreover
other loss functions do not yield a recurrence with 
a closed form solution. 
To overcome these problems, we use (\ref{eq:closed1}) 
as a starting point to think about the consequences of 
presenting an example many times. To compensate, we will 
also need to adjust the learning rate we use everytime 
we present an example. Suppose that we present an example 
a factor of $n$ times more using a learning rate that 
is smaller by a factor of $n$.  This can be simulated in 
constant time using (\ref{eq:closed1}) with $hn$ and 
$\frac{\eta}{n}$ in place of $h$ and $\eta$ respectively. 
Letting $n$ grow large, we are interested in
\[  
\lim_{n\to \infty}\frac{w_t^\top x -y}{x^\top x}\left(1-\left(1-\frac{\eta x^\top x}{n} \right)^{nh}\right)
\]
Using that $\lim_{n\to\infty} (1+z/n)^n =e^z$ we have that
\begin{equation}\label{eq:deriv1}
s(h)=\frac{w_t^\top x-y}{x^\top x}\left(1-\exp (-h\eta x^\top x)\right).
\end{equation}
 
The key in the above derivation is 
using the limit of the gradient descent process
as the learning rate becomes infinitesimal.
We now generalize this idea to derive updates for 
other loss functions. 
\begin{thm}\label{mainthm}
The limit of the gradient descent process 
as the learning rate becomes infinitesimal
for an example with importance weight $h \in \mathbb{R}_+$
is equal to the update
\[
 w_{t+1} = w_t - s(h) x 
\]
where the scaling factor $s(h)$ satisfies the differential equation:
\begin{equation}\label{eq:diffeq}
 s'(h) = \eta \left.\frac{\partial \ell}{\partial p}\right|_{p=(w_t -s(h)x)^\top x}, \quad s(0) = 0
\end{equation}
\end{thm}
The proof is in the appendix.
This theorem is our framework for deriving updates for many loss functions.
Plugging a loss function in (\ref{eq:diffeq}) gives an ODE whose solution
is the result of a continuous gradient descent process. The ODE 
can be easily solved by separation of variables. 

As a sanity check for squared loss, $\frac{\partial \ell}{\partial p}=p-y$, (\ref{eq:diffeq}) gives
\[
 s'(h) = \eta ((w_t -s(h)x)^\top x-y), \quad s(0) = 0
\]
a linear ODE, whose solution exactly rederives (\ref{eq:deriv1}).

\subsection{Other Loss Functions} 

\begin{table*}
\caption{\label{tab:losses}Importance Invariant and Imp$^2$ (cf. section~\ref{imp2}) Updates for Various Loss Functions}
\centering{
\begin{tabular}{|c|c|c|c|}
\hline 
Loss & $\ell(p,y)$ & Invariant Update $s(h)$ & Imp$^2$ Update\tabularnewline
\hline
\hline 
Squared & $\frac{1}{2}(y-p)^{2}$ & $\frac{p-y}{x^{\top}x}\left(1-e^{-h\eta x^{\top}x}\right)$ & $\frac{h\eta(p - y)}{1+ h\eta x^{\top}x}$\tabularnewline
\hline 
Logistic & $\log(1+e^{-yp})$ & $\frac{W(e^{h\eta x^{\top}x+yp+e^{yp}})-h\eta x^{\top}x-e^{yp}}{yx^{\top}x}$
for $y\in\{-1,1\}$ & Not Closed \tabularnewline
\hline 
Exponential & $e^{-yp}$ & $\frac{py-\log(h\eta x^{\top}x+e^{py})}{x^{\top}xy}$ for $y\in\{-1,1\}$ & Not Closed \tabularnewline
\hline 
Logarithmic & $y\log\frac{y}{p}+(1-y)\log\frac{1-y}{1-p}$ & $\begin{array}{cc}
\mbox{if } y=0 & \frac{p-1+\sqrt{(p-1)^2+2h\eta x^\top x}}{x^\top x}\\
\mbox{if } y=1 & \frac{p-\sqrt{p^2+2h\eta x^\top x}}{x^\top x} \end{array}$ & Not Closed \tabularnewline
\hline 
Hellinger &
$2(1-\sqrt{py}-\sqrt{(1-p)(1-y)})$
& $\begin{array}{cc}
\mbox{if } y=0 & \frac{p-1+\frac{1}{4}(12h\eta x^{\top}x+8(1-p)^{3/2})^{2/3}}{x^{\top}x}\\
\mbox{if } y=1 & \frac{p-\frac{1}{4}(12h\eta x^{\top}x+8p{}^{3/2})^{2/3}}{x^{\top}x} \end{array}$ & Not Closed\tabularnewline
\hline 
Hinge & $\max(0,1-yp)$ & $-y\min\left(h\eta,\frac{1-yp}{x^{\top}x}\right)$ for $y\in\{-1,1\}$ & Same\tabularnewline
\hline 
$\tau$-Quantile & $\begin{array}{cc}
\mbox{if } y>p & \tau(y-p)\\
\mbox{if } y\leq p & (1-\tau)(p-y)
\end{array}$  & 
$\begin{array}{cc}
\mbox{if } y>p & -\tau\min(h\eta, \frac{y-p}{\tau x^\top x}) \\
\mbox{if } y\leq p & (1-\tau) \min(h\eta, \frac{p-y}{(1-\tau) x^\top x})
\end{array}$ & Same \tabularnewline
\hline
\end{tabular}
}
\end{table*}

Using (\ref{eq:diffeq}) as our framework, we can derive step 
sizes for many popular loss function as summarized in
table~\ref{tab:losses}. 

For the logistic loss, the solution involves the
Lambert W function: $W(z)e^{W(z)}=z$, and the 
solution can be verified using $W'(z)=\frac{W(z)}{z(1+W(z))}$.
The exponential loss also fits
nicely into our framework.

For the logarithmic loss
the ODE has no explicit form for all $y \in [0,1]$. 
The table presents the common case $y \in \{0,1\}$. In this case 
each value of $y$ gives rise to an ODE whose solution has an explicit form.
Note that here the ODE solutions
satisfy a second degree equation and hence each branch 
has two solutions. We have selected the one 
satisfying $s'(0)=\eta \frac{\partial \ell}{\partial p}$.
To avoid an infinite loss, we can clip the predictions 
away from 0 and 1 (and update using $\min(h,h')$ 
where $h'$ is the importance weight that hits the clipping point) 
or use a link function such as $\tanh$. 
In the appendix we show updates for this latter case.

A similar situation arises for the Hellinger loss.
The solution to (\ref{eq:diffeq})  has no simple form 
for all $y \in [0,1]$ but for $y \in \{0,1\}$ we get
the expressions in table~\ref{tab:losses}.

\subsubsection{Hinge Loss and Quantile Loss}

Two other commonly used loss function are the 
hinge loss and the $\tau$-quantile loss 
where $\tau \in [0,1]$ is a parameter. These are differentiable everywhere except 
at one point where the subdifferential contains zero.
 
Hence, for the hinge loss, a valid expression for (\ref{eq:diffeq}) is
\[
s'(h)=
\begin{cases}
-\eta y & y(w-s(h)x)^\top x<1\\
0  & y(w-s(h)x)^\top x\geq 1
\end{cases}
\] 
The first branch (together with $s(0)=0$) gives $s(h)=-yh \eta$
for $y(w+yh\eta x)^\top x<1$. Otherwise, i.e. when 
$h \geq h_{\textrm{hinge}} = \frac{1-y w^\top x}{\eta x^\top x}$,
$s(h)$ is a constant. Here $h_{\textrm{hinge}}$ is the importance 
weight that would make the updated prediction lie at the hinge. 
To maintain continuity at $h_{\textrm{hinge}}$
we set $s(h)=-yh_{\textrm{hinge}} \eta$. In conclusion
\[
s(h)=-y\min(h,h_{\textrm{hinge}})\eta
\]
This matches the intuition when one thinks
about the limit of infinitely many infinitely 
small updates: For large importance weights,
the process will bring the prediction up to 
$y$ and make no further progress. 

The quantile loss is similar and the update rule 
first computes the importance weight $h'$ that would take the 
updated prediction at the point of nondifferentiability and then 
multiplies the gradient by $\min(h,h')$. 

\subsection{Variable Learning Rate} \label{variable-eta}
To handle a decaying learning rate $\eta_t$, we just need to 
slightly modify (\ref{eq:diffeq}). Let $\eta_t(u)$ be the value of 
the learning rate $u$ timesteps after time $t$. Then (\ref{eq:diffeq}) becomes
\[
s'(h)= \eta_t(h) \left.\frac{\partial \ell}{\partial p}\right|_{p=(w_t -s(h)x_t)^\top x_t}, \quad s(0) = 0
\]
The solutions in this case are not qualitatively any different from 
the solutions of (\ref{eq:diffeq}). We just need to replace 
the occurrences of $h\eta$ with  
$\int_0^h \eta_t(u) du$. Again for popular choices of learning 
rate such as $\eta_t(u)=\frac{1}{(t+u)^p}$ with $p=\frac{1}{2}$ or $p=1$,
this has a closed form.

\subsection{Regularization}

Theorem~\ref{mainthm} can be modified to 
handle losses of the form $\ell(w^\top x,y) + \frac{\lambda}{2} ||w||^2$
where $||w||$ is the Euclidean norm of $w$.
However, the resulting differential equation is considerably harder
and we have only been able to obtain a solution for the case of squared loss.
Therefore we describe 
an alternative way of incorporating regularization based on
a splitting approach
\cite{duchi2009efficient}. First perform $h$ 
unconstrained steps using the closed form solution,
then compute the effect of regularization:
\[
w_{t+1} = \argmin_w \frac{1}{2}||w - (w_t -s(h)x)||^2 + \frac{h\eta\lambda}{2} ||w||^2.
\]
Note that we apply all $h$ regularizers at once.
The solution to the above optimization problem is
\[
w_{t+1} =  \frac{w_{t} - s(h) x_t}{1+h\eta\lambda}
\]
This approach can also handle other regularizers 
such as $\lambda||w||_1$ leading to 
a truncated gradient update \cite{langford2009sparse}.

\section{PROPERTIES OF THE UPDATES} \label{properties}
\subsection{Invariance}
First we show that the invariance property 
we mentioned in the introduction holds.
It is convenient to explicitly state the 
dependence on the prediction $p$, by writing
$s(p,h)$ instead of $s(h)$.
The following theorem states that 
an update with an importance weight $a+b$ is equivalent to
an update with importance $a$ immediately followed by an
update with importance $b$.
\begin{thm} \label{thm:invariance} 
Let $s(p,h)$ be the solution of
\[
\frac{\partial s}{\partial h} = \eta \left.\frac{\partial \ell}{\partial p}\right|_{p=w^\top x -s(p,h) x^\top x}, \quad s(p,0) = 0
\]
where $\ell$ is a continuously differentiable loss. Then
\[
s(p,a+b)=s(p,a)+s(p-s(p,a)x^\top x,b)
\]
\end{thm}
The proof is in the appendix and 
uses the Existence and Uniqueness Theorem for ODEs. 

\subsection{Safety} 
For some loss functions such 
as squared loss, hinge loss and quantile loss the residual 
$w_t^\top x_t-y_t$ tells us whether the learner is overestimating
or underestimating the target. We call an update safe if
\[
\frac{w_{t+1}^\top x_t-y_t}{w_t^\top x_t-y_t}\geq 0
\]
whenever $(w_t^\top x_t-y_t)\neq 0$. Since the 
residual does not change sign after a safe update,
it leads to sane results even when the learning rate is 
very aggressive. 

Standard gradient descent is not safe, and importance invariant 
step sizes always are. This should be obvious for the hinge loss and 
the quantile loss because they use the minimum necessary step.
For squared loss:
\begin{eqnarray*}
\frac{w_{t+1}^\top x -y}{w_{t}^\top x -y} & = &
\frac{w_t^\top x + \frac{y-w_t^\top x}{x^\top x}\left(1-e^{-h\eta x^\top x}\right)x^\top x -y}{w_{t}^\top x -y} \\
& = & e^{-h \eta x^\top x} > 0
\end{eqnarray*}
hence the update is safe.

\subsection{Fallback Regret Analysis}

Here we provide a fallback analysis for the case $h_t=1$.
For simplicity we only show the results for squared loss and 
$||x_t||=1$ for all $t$. However, this can be extended to other losses as
a Taylor expansion of each update around $\eta = 0$ shows that 
to first order, it is equivalent to online 
gradient descent. Hence we expect a regret 
analysis similar to the one achieved by the underlying 
learning rate schedule. A proof of the following theorem
is in the appendix. 
\begin{thm}\label{regretthm}
If $\ell(p,y)=(p-y)^2$ and $||x_t||=1$ for all $t$ then
the importance invariant update attains a regret of
$O(\sqrt{T})$ when 
$\eta_t=\frac{1}{\sqrt{t}}$, and a regret of $O(\log(T))$ when 
$\eta_t=\frac{1}{t}$.
\end{thm}

\begin{figure*}[ht]
\includegraphics[width=0.475\textwidth]{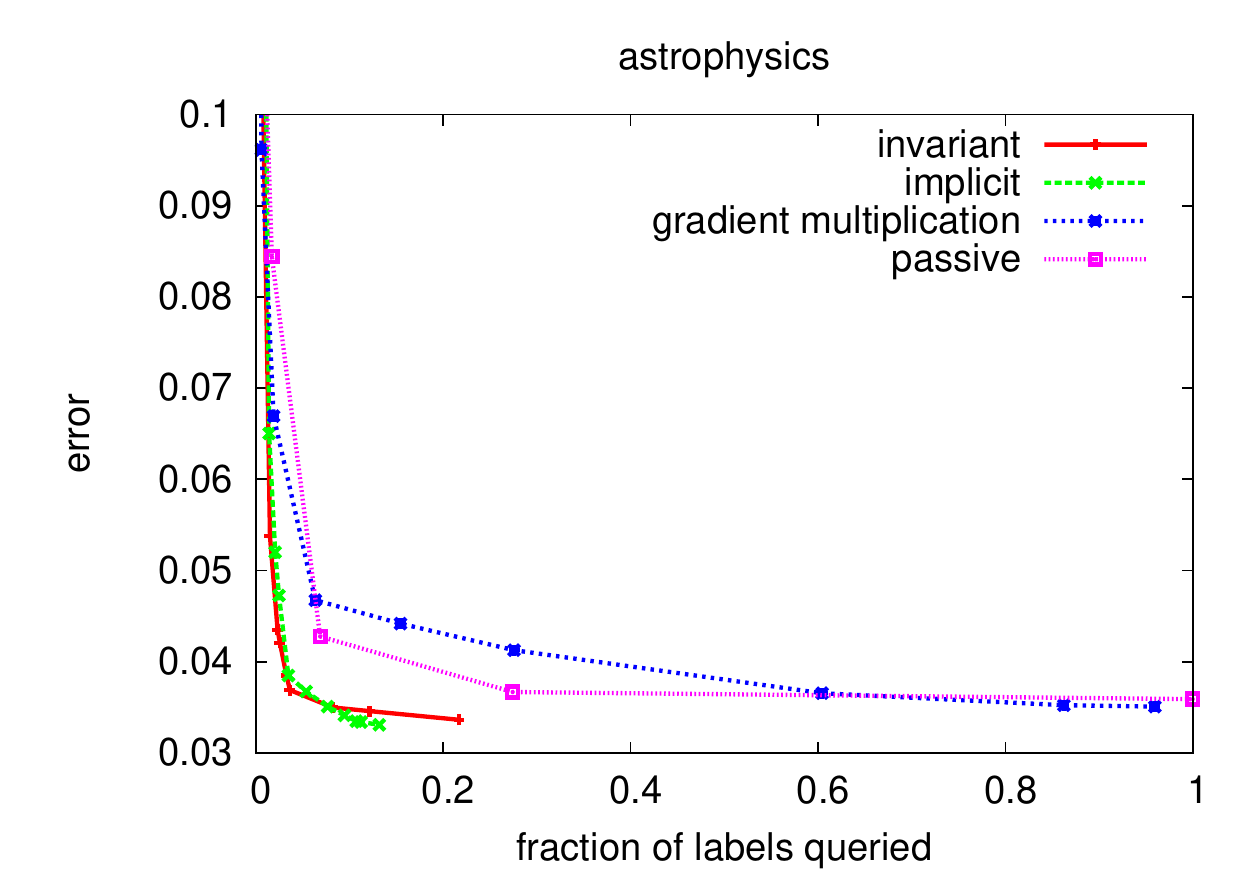}
\includegraphics[width=0.475\textwidth]{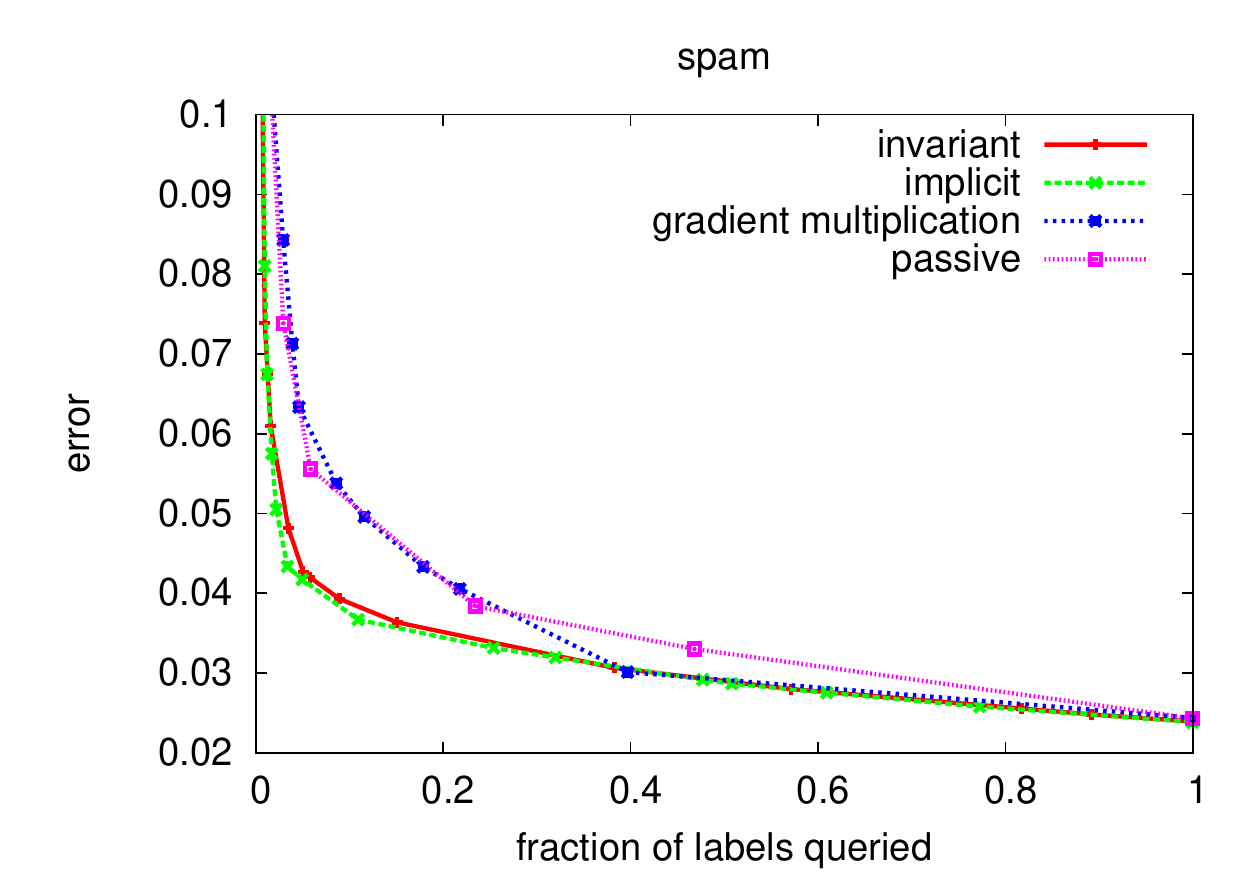}\\
\includegraphics[width=0.475\textwidth]{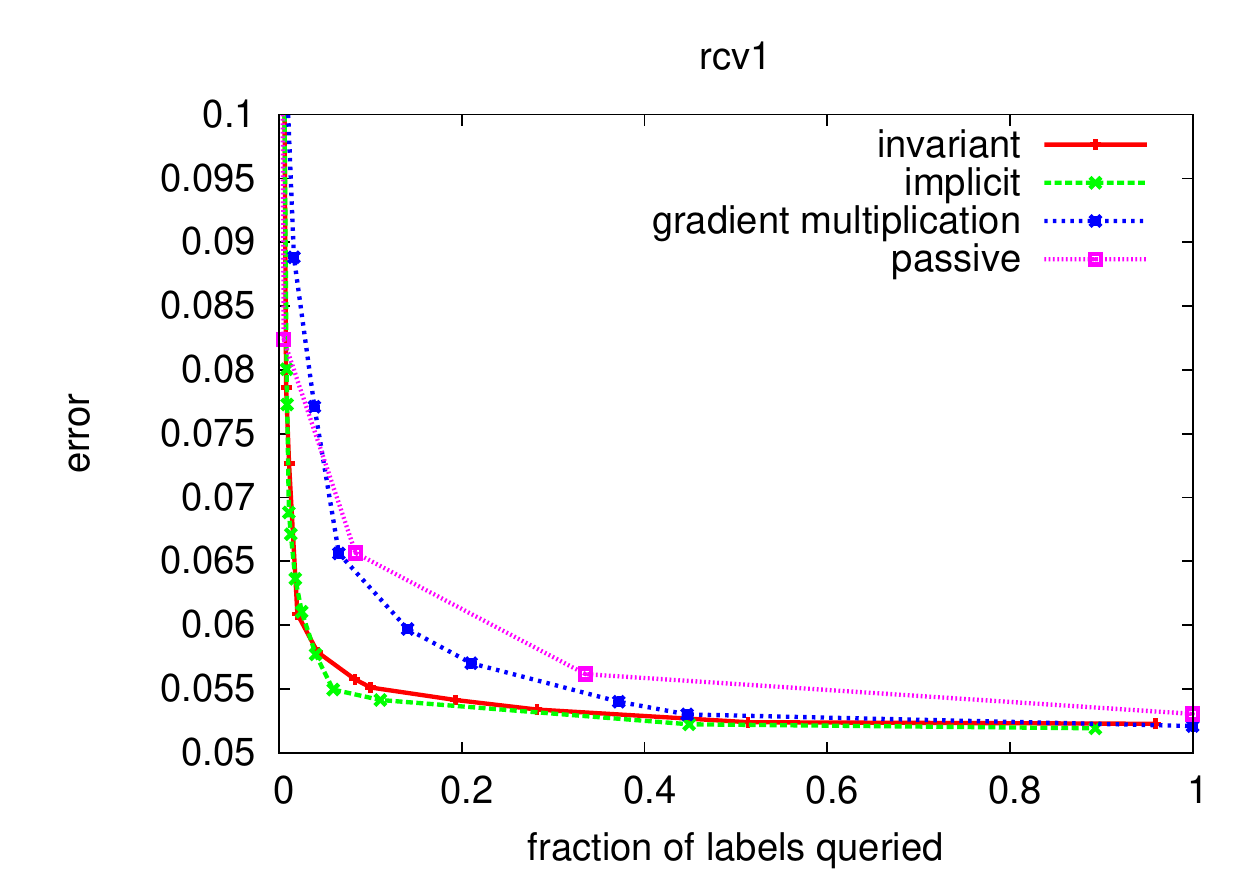}
\includegraphics[width=0.475\textwidth]{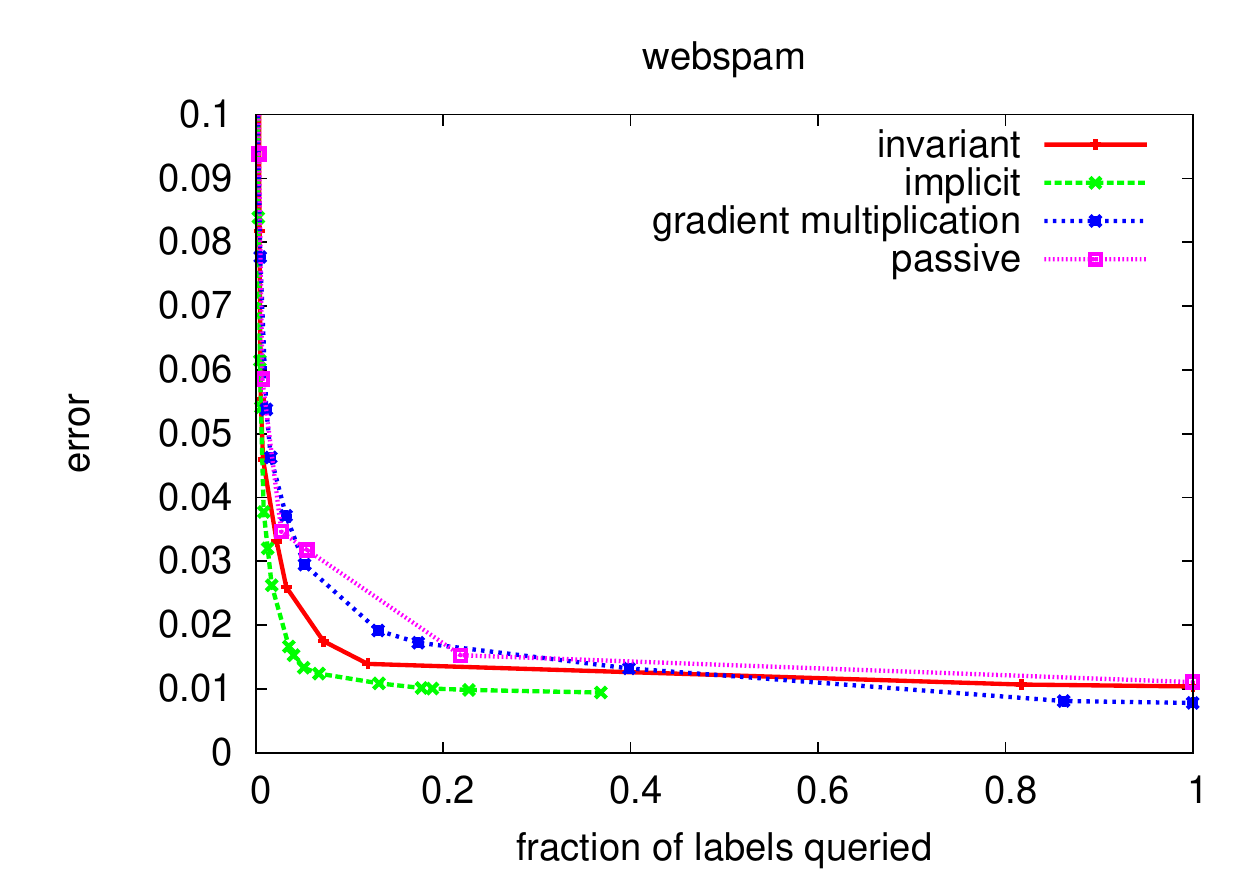}\\
\caption{\label{fig:active} Test error vs. fraction of queried labels for each dataset} 
\end{figure*}

\section{IMPLICIT IMPORTANCE WEIGHT UPDATES}\label{imp2}

Implicit updates, first proposed in~\cite{ig1} and recently 
analyzed in~\cite{ig2} provide an alternative way for handling 
importance weights. An implicit update sets: 
$$w_{t+1}=\argmin\frac{1}{2}||w-w_{t}||^{2}+\lambda\ell(w^\top x,y)$$ where
$\lambda$ is a free parameter similar to the learning rate in
interpretation.  Finding the minimizing $w$ generally requires an
iterative root-finding algorithm which is perhaps an order of
magnitude more expensive than the closed form updates we derive above,
although often easily amortized by the update itself. For squared loss
and hinge loss closed-form solutions have been known. In the appendix 
we show how to derive these as well as a closed form implicit 
update for quantile loss which, to the best of our knowledge, is new.
To adapt implicit updates for importance weights we simply use $\lambda_t =
\eta h_t$ (or $\lambda_t = \int_0^{h_t}\eta_t(u)du$ as in section \ref{variable-eta}) 
yielding an 
algorithm we call Imp$^2$. Imp$^2$ has qualitatively similar properties, 
satisfying Safety and Regret, but not Invariance or a Closed form update.  
In fact, Imp$^2$ for hinge and quantile is precisely equivalent to the 
importance invariant update. 
\vspace{-0.3cm}
\section{EXPERIMENTS} \label{exps}
\vspace{-0.25cm}
We present empirical results on four text 
classification datasets:
`rcv1' is a modified version \cite{LanLiStr07} of RCV1 \cite{lewis2004rcv1},
`astro' is from \cite{joachims2006training},
`spam' was created from the TREC 2005 spam public corpora,
and `webspam' is from the PASCAL large scale learning challenge.
In all experiments we did a single pass through the 
training set and we report the error on the
test set. We try all learning rate schedules
of the form $\eta_t = \frac{\mu}{x_t^\top x_t} \left(\frac{\tau}{t+\tau}\right)^p$
with $(\mu,\tau,p) \in \{2^i\}_{i=0}^{10} \times \{10^i\}_{i=0}^8 \times \{0.5,1\}$.

In the first set of experiments, large importance weights will inevitably appear.
In particular, we treat all four datasets as active learning tasks
and apply the algorithm in \cite{beygelzimer2010agnostic}.
At each timestep $t$ this algorithm computes a probability of querying the label of a 
point based on a quantity $G_t$ that measures the difference 
in error rates between two hypotheses. The empirical risk minimizing (ERM) hypothesis and
an alternative hypothesis that minimizes the empirical risk subject
to predicting a different label than the ERM hypothesis. 
In our case the hypothesis we are learning with online gradient descent 
acts as the ERM hypothesis. 
To get a handle on $G_t$ we estimate $t\cdot G_t$ by the importance weight
that the example would need to have in order for an update with the
alternative hypothesis's prefered label to cause the classification of the
example to become the alternative label. In the appendix
we derive the relevant importance weights for invariant and implicit updates.
Once we have an estimate for $G_t$ we can compute the query probability
$P_t$ and, if we query the label, we add the example to our dataset with 
importance weight $1/P_t$ to keep things unbiased. The query probabilities
turn out to be proportional to the learning rate $\eta_t$ and hence 
the algorithm will generate importance weights 
of order $O(\eta_t^{-1})$ growing at least as fast as $\Omega(\sqrt{t})$.
Notice that importance weights are used for two different tasks: 
estimating $G_t$ and preserving unbiasedness.
Our linear model (with a link function 
$\sigma(p)=\max(0,\min(1,p))$) is optimizing squared loss. 
Details are in the appendix.

\begin{figure}[t]
\centering{
\includegraphics[width=0.475\textwidth]{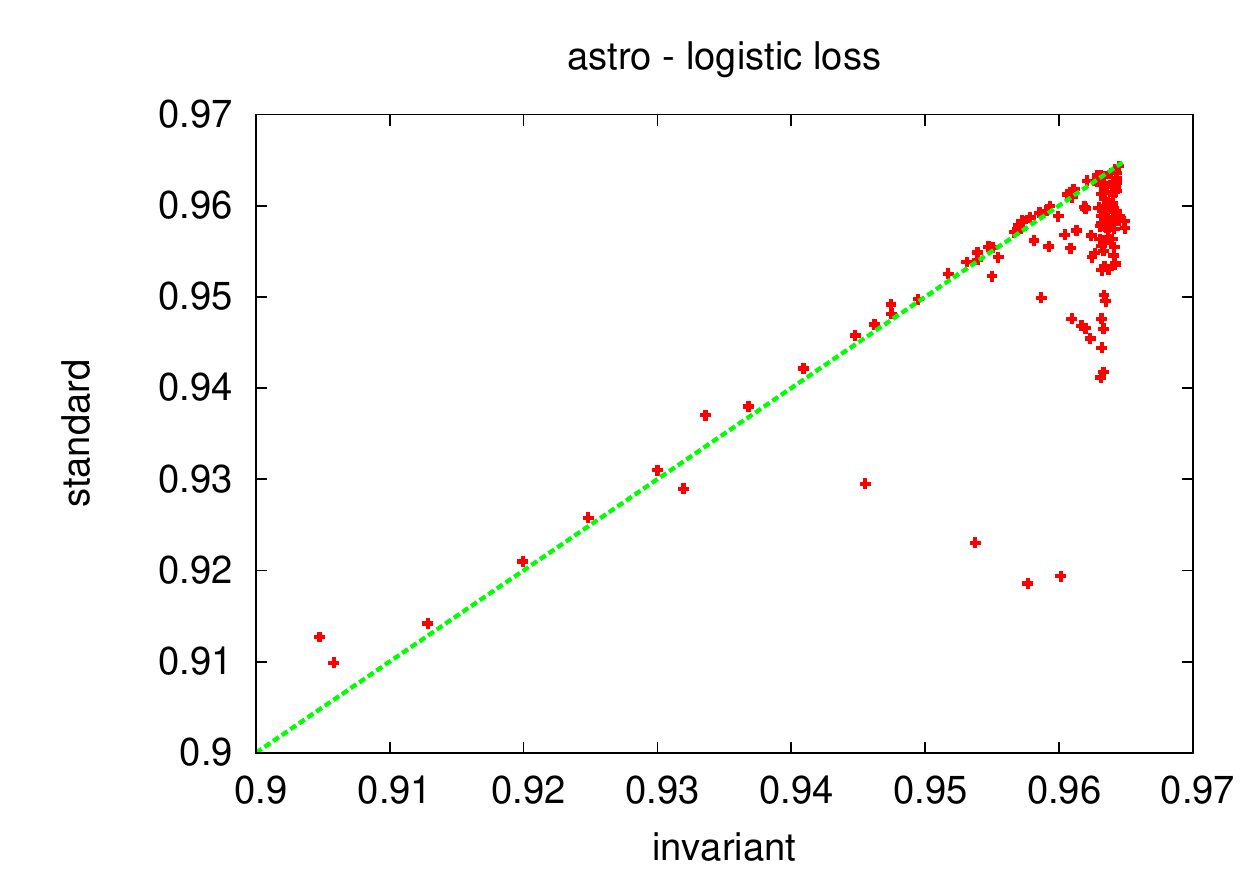}
}
\caption{Scatter plot showing test accuracy with two different updates for various datasets and losses} \label{fig:scatter}
\end{figure}

In Figure~\ref{fig:active} we summarize the results of the active
learning experiments. Each combination of learning rate schedule and
setting of the parameter $C_0$ in the active learning algorithm ($C_0
\in \{10^{-8},10^{-7},\ldots,10^1\}$) is an experiment that can be
represented in the graph by a point whose $x$-coordinate is the
fraction of labels queried by the active learning algorithm and whose
$y$-coordinate is the test error of the learned hypothesis. To
summarize this set of points, the figures plot part of its convex
hull. The points on the convex hull (sometimes called a Pareto
frontier) are experiments which represent optimal tradeoffs between
generalization and label complexity, for some setting of this
tradeoff. When a curve stops sooner than the size of the dataset it
means that there were no experiments in which using more queries gave
better generalization.  We have also included the results from a
typical good run of a passive learner.  The graphs show very
convincingly the value of having an update that handles importance
weights correctly. Doing so yields better generalization and
lower label complexity, than those attainable by multiplying the
gradient with the importance weight. In fact, table~\ref{tab:reduction}
which lists the ratio of labels between passive and active learning 
to achieve a given accuracy, shows that \emph{linearization can 
make active learning need more labels than passive learning}.

In the second set of experiments we used the same datasets but treated 
all examples as having an importance weight of one. We compare 
standard online gradient, vs. invariant and implicit updates
on four loss functions: squared,
logistic, hinge and quantile($\tau=0.5$) loss. The purpose of these
experiments is to highlight the robustness of the invariant updates:
they yield good generalization with little search for a
good learning rate schedule
(also noted in \cite{ig2} for implicit updates). However, we begin with
table~\ref{tab:exhaustive} which shows the test accuracy of the
hypotheses learned by each update after \emph{exhaustively
  searching} over the learning rate schedule. For astro and
rcv1 the differences are very small. The spam dataset is not TF-IDF 
processed, and there we see a substantial improvement with either new update. The
results on the webspam dataset were initially puzzling, but we have
verified that this is not a failure to optimize well; on the contrary
the proposed updates \emph{attain smaller progressive validation
  loss}~\cite{BKL99} than standard online gradient descent on the
training data.  Since progessive validation loss deviates like a test
set, this is evidence that the webspam test set has a different
distribution from the training set.\footnote{The test set consists of
  the last 50000 examples of the original training set. The
  real test labels are not public.}

\begin{table}
\caption{Reduction in label complexity}\label{tab:reduction}
\centering{
\begin{tabular}{ccccc}
\hline
Dataset & astro  & rcv1  & spam  & web \\
Desired Accuracy & 0.963 & 0.943 & 0.967 & 0.986 \\
\hline
Multiplication & 0.45 & 1.59 & 1.28 & 0.54 \\
Implicit & 5.12 & 6.55 & 1.88 & 4.33 \\
Invariant& 7.56 & 6.55 & 2.13 & 1.82 \\
\hline
\end{tabular}
}
\end{table}

\begin{table}
\caption{Test accuracies (grid search over schedules)}\label{tab:exhaustive}
\begin{tabular}{ccccc}
\hline
Dataset & Loss & Invariant & Imp$^2$ & Standard\\
\hline
astro & hinge &  0.96626 & Same & 0.96694\\
 & logistic &  0.96494 & 0.96485 & 0.96432\\
 & quantile &  0.96629 & Same & 0.96703\\
 & squared &  0.96463 & 0.96429 & 0.96469\\
\hline
rcv1 & hinge &  0.94872 & Same & 0.94838\\
 & logistic &  0.94704 & 0.94682 & 0.94743\\
 & quantile &  0.94846 & Same & 0.94859\\
 & squared &  0.94769 & 0.94799 & 0.94790\\
\hline
spam & hinge &  0.97626 & Same & 0.97411\\
 & logistic &  0.96676 & 0.97982 & 0.97603\\
 & quantile &  0.97524 & Same & 0.97484\\
 & squared &  0.97609 & 0.97614 & 0.97563\\
\hline
web & hinge &  0.98936 & Same & 0.99142\\
 & logistic &  0.99094 & 0.99038 & 0.9923\\
 & quantile &  0.98908 & Same & 0.99088\\
 & squared &  0.98960 & 0.98966 & 0.99218\\
\hline
\end{tabular}
\end{table}

To illustrate robustness we present the results in two ways. First, in
Figure~\ref{fig:scatter} we show a scatterplot where each point is
a learning rate schedule and its coordinates are the accuracy of the
learned hypothesis with and without the proposed updates.  Scatter
plots for other loss functions and datasets look very similar and are
included in the appendix.
The plot only shows the cases when both learning rates achieve
accuracy above $0.9$ and there are virtually no schedules for which
one update is superior by more than $0.1$. Among these cases the vast
majority of experiments are clustered under the $y=x$ line and towards
the extreme values of the $x$-axis. Consequently, when using the
importance invariant update many schedules provide excellent
performance.

\begin{table}
\caption{Fraction of schedules with near optimal error}\label{tab:fractions}
\centering{
\begin{tabular}{ccc}
\hline
Loss & Invariant & Standard\\
\hline
hinge &  0.337 & 0.039\\
logistic & 0.109 & 0.050\\
quantile & 0.361 & 0.053\\
squared & 0.306 & 0.031\\
\hline
\end{tabular}
}
\end{table}

To make this more clear we present a second way of
viewing this result. In table~\ref{tab:fractions} we report the 
fraction of learning rate schedules that achieve
generalization accuracy within $0.001$ of the best
learning rate schedule, on average across all 
four datasets. For loss functions for which 
there is a notion of overshooting and can 
benefit from a safe update we observe an order 
of magnitude improvement in the number of 
schedules that converge to near optimal performance. 
\vspace{-0.25cm}
\section{CONCLUSIONS} \label{concl}
\vspace{-0.25cm}
We tuned online gradient descent learning algorithms for various
losses so they efficiently incorporate importance weight information,
as is needed for applications in boosting, active learning, transfer
learning, and learning reductions.  The essential lesson here is that
taking into account the curvature of the loss function can be done cheaply
and provides great benefits in dealing with importance weights.

Motivated by an invariance property we proposed new updates that improve 
the standard update rule even for the
baseline importance weight $1$ case, yielding better prediction
performance while simultaneously reducing the value of learning rate
parameter search. Experiments not reported here show that it even
improves the performance of adaptive gradient descent methods such as
\cite{brendan2010adaptive,duchi2010adaptive}. Since this tuned update
rule is computationally ``free'' we expect wide use.
\vspace{-0.25cm}
\subsubsection*{Acknowledgements}
We thank Robert Kleinberg for help with the proof of Theorem~\ref{thm:invariance} 
and Alex Strehl for insightful discussions.
\vspace{-0.25cm}
\bibliography{bibfile}
\appendix
\section*{APPENDIX}

\section{PROOF OF MAIN THEOREM}
In this section we prove Theorem \ref{mainthm}:
\begin{unthm}
The limit of the gradient descent process 
as the learning rate becomes infinitesimal
for an example with importance weight $h \in \mathbb{R}_+$
is equal to the update
\[
 w_{t+1} = w_t - s(h) x 
\]
where the scaling factor $s(h)$ satisfies the differential equation:
\begin{equation}\label{eq:diffeq2}
 s'(h) = \eta \left.\frac{\partial \ell}{\partial p}\right|_{p=(w_t -s(h)x)^\top x}, \quad s(0) = 0
\end{equation}
\end{unthm}
\begin{proof}
In accordance to the proof of 
lemma~\ref{lem:int}, we can compute the effect
of an importance weight of $h+\epsilon$, assuming
we know the effect of an importance weight of $h$,
by performing an additional gradient step with the
learning rate appropriately scaled by $\epsilon$:
\[
 s(h+\epsilon) = s(h)+\epsilon \eta \left.\frac{\partial \ell}{\partial p}\right|_{p=(w_t -s(h)x)^\top x}.
\]
Rearranging we have
\[
 \frac{s(h+\epsilon)- s(h)}{\epsilon} = \eta \left.\frac{\partial \ell}{\partial p}\right|_{p=(w_t -s(h)x)^\top x}.
\]
Taking the limit as $\epsilon \to 0$ gives the result.
The initial condition makes sure that a zero importance weight has no effect in the update.
\end{proof}

\section{PROOF OF INVARIANCE}
In this section we prove Theorem \ref{thm:invariance}. For convenience
let $f(z)=\eta\left.\frac{\partial \ell}{\partial p}\right|_{p=z}$ and $n=x^\top x$.
The theorem says that if $f$ is continuous and $s$ satisfies
\begin{align*}
\frac{\partial s}{\partial h} & = f(p-ns(p,h)) \\
s(p,0) & = 0
\end{align*}
then
\begin{equation} \label{eq:functional-equation}
s(p,a+b) = s(p,a) + s(p-ns(p,a),b).
\end{equation}
\begin{proof}
Fix an arbitrary value of $a$,
and consider the pair of single-variable functions
\begin{align*}
u(b) & = s(p,a+b) \\
v(b) & = s(p,a) + s(p-ns(p,a),b).
\end{align*}
These satisfy 
\[
u(0) = v(0) = s(p,a)
\]
and they satisfy
\begin{align*}
\frac{du}{db} & = 
\left( \frac{\partial s }{\partial h} \right)_{(p,a+b)} =
f(p - ns(p,a+b)) = f(p - nu(b)) \\
\frac{dv}{db} & =
\left( \frac{\partial s}{\partial h} \right)_{(p-ns(p,a),b)} = \\
 & = f(p - ns(p,a) - ns(p-ns(p,a),b)) =
f(p - nv(b)).
\end{align*}
In other words, both $u$ and $v$ are solutions of the 
ordinary differential equation
\[
\frac{dw}{db} = f(p - nw(b))
\]
with initial condition $w(0) = s(p,a)$.  Since 
$f$ satisfies the hypotheses of the existence and
uniqueness theorem for ordinary differential equations,
it is valid to conclude that $u(b) = v(b)$, which
verifies~\eqref{eq:functional-equation}.
\end{proof}

\section{FALLBACK REGRET PROOF}
In this section we prove Theorem \ref{regretthm}:
\begin{unthm}
If $\ell(p,y)=(p-y)^2$ and $||x_t||=1$ for all $t$ then
the importance invariant update attains a regret of
$O(\sqrt{T})$ when 
$\eta_t=\frac{1}{\sqrt{t}}$, and a regret of $O(\log(T))$ when 
$\eta_t=\frac{1}{t}$.
\end{unthm}

\begin{proof} (Sketch) First, consider the case when
$\eta_t = \frac{1}{\sqrt{t}}$ and therefore the step sizes
are $\eta'_t=1-\exp(-\frac{1}{\sqrt{t}})$. In general the regret of online 
gradient descent depends on the step sizes via \cite{Zinkevich03}
\[
\frac{1}{\eta'_T}+\sum_{t=1}^T \eta'_t =
\frac{1}{1-\exp(-\frac{1}{\sqrt{T}})} + \sum_{t=1}^T (1-\exp(-\frac{1}{\sqrt{t}}))
\]
The Taylor exansion of $\frac{1}{1-\exp(-\sqrt{u})}$ around zero suggests 
that the first term grows asymptotically as
$\sqrt{T}+\frac{1}{2}+O(\frac{1}{\sqrt{T}})$.
The second term can be bounded by $1+\int_{1}^T (1-\exp(-\frac{1}{\sqrt{t}}))dt$
which evaluates to 
\[
\sqrt{T} \exp(-\frac{1}{\sqrt{T}}) +T(1-\exp(-\frac{1}{\sqrt{T}}))-\Gamma(0,\frac{1}{\sqrt{T}})+\Gamma(0,1)
\] 
where $\Gamma(a,x)=\int_x^\infty u^{a-1}e^{-u}du$ is the incomplete 
gamma function, with the property $\Gamma(0,\frac{1}{u})=O(\log(u))$ for 
large $u$. The first two terms are both $O(\sqrt{T})$, the third 
term is $O(\log(\sqrt{T}))$ and the last term is constant.
Hence the regret is $O(\sqrt{T})$, as is the regret we would attain with
$\eta_t=\frac{1}{\sqrt{t}}$.

Now, we consider the case $\eta_t=O(\frac{1}{t})$. This type of schedule 
is known to give logarithmic regret for strongly convex functions such as 
squared loss. The same holds true for our update rule.  
In general, for a 1-strongly convex loss such as squared
loss, regret depends on \cite{shalev2007logarithmic}
\[
\sum_{t=1}^T\left(\frac{1}{\eta'_t}-\frac{1}{\eta'_{t-1}} - 1\right)
+\sum_{t=1}^T \eta'_t=
\]
\[
\sum_{t=1}^T\left(\frac{1}{1-e^{-\frac{1}{t}}}-\frac{1}{1-e^{-\frac{1}{t-1}}} -1\right)
+\sum_{t=1}^T \left(1-e^{-\frac{1}{t}}\right)
\]
By computing the Taylor series of the first term, we see that 
each individual summand is $O(\frac{1}{t^2})$ and therefore 
the first sum is bounded by a constant. 
As before, the second sum can be bounded by
\[
1+T(1-e^{-\frac{1}{T}})+\Gamma(0,\frac{1}{T}).
\]
The second term is less than $1$ and the third term
is $O(\log(T))$, which is the regret attained by the original $\eta_t=\frac{1}{t}$ update.
\end{proof}

\section{EXPLICIT IMPLICIT UPDATES}

Here we derive closed form updates for Imp$^2$ for some losses. The results
on squared loss and hinge loss are known, but the result quantile loss is new.
\subsection{Squared Loss}
For the squared loss we have
\[
w_{t+1}=\arg\min\frac{1}{2}||w-w_{t}||^{2}+\frac{\lambda}{2}(w^{\top}x-y)^{2}
\]
Taking the derivative of the objective with respect to $w$ and setting
to zero we get
\begin{equation}
w=w_{t}-\lambda(w^{\top}x-y)x\label{eq:sqimp1}
\end{equation}
We can solve for $w^\top x$ by taking the inner product of both sides with $x$:
\[
w^{\top}x=w_{t}^{\top}x-\lambda(w^{\top}x-y)x^{\top}x
\]
 \[
w^{\top}x=\frac{w_{t}^{\top}x+\lambda yx^{\top}x}{1+\lambda x^{\top}x}
\]
Now we substitute back in (\ref{eq:sqimp1}) and we get
\[
w=w_{t}-\frac{\lambda(w_{t}^{\top}x-y)}{1+\lambda x^{\top}x}x
\]

\subsection{Hinge Loss}
The implicit update for the hinge loss comes from the solution of
\begin{eqnarray*}
w_{t+1} & = & \arg\min\frac{1}{2}||w-w_{t}||^{2}+\lambda\xi\\
\mathrm{s.t.} &  & \xi\ge0\\
 &  & \xi\geq1-yw^{\top}x
\end{eqnarray*}
This update has been considered by \cite{crammer2006online} under
the name PA-I. They show that the solution to the above optimization problem is 
\[
w_{t+1}=w_{t}+y\min(\lambda,\frac{1-yw_{t}x}{x^{\top}x})x
\]
which is the same as our importance invariant update. However, 
they treat $\lambda$ as a fixed hyperparameter, while here
$\lambda=h_t\eta_t$ is varying as $t$ varies. 

\subsection{Quantile Loss}
For the quantile loss, we can write the update as:
\begin{eqnarray*}
w_{t+1} & = & \arg\min\frac{1}{2}||w-w_{t}||^{2}+\lambda\xi\\
\mathrm{s.t.} &  & \xi\ge\tau(y-w^{\top}x)\\
 &  & \xi\ge(1-\tau)(w^{\top}x-y)
\end{eqnarray*}
To find $w_{t+1}$ we introduce the Lagrangian:
\begin{eqnarray*}
L(w,\xi,\mu)& = & \frac{1}{2}||w-w_{t}||^{2}+\lambda\xi+\\
& & +\mu_{1}(\tau(y-w^{\top}x)-\xi)+\\
& & +\mu_{2}((1-\tau)(w^{\top}x-y)-\xi)
\end{eqnarray*}
with Lagrange multipliers $\mu_{1}\geq0$, $\mu_{2}\ge0$. Setting
the derivative of $L$ to zero w.r.t. $w$ and $\xi$ we get
\begin{equation}\label{eq:lagrangian_quantile}
w=w_{t}+\mu_{1}\tau x-\mu_{2}(1-\tau)x
\end{equation}
\begin{equation}\label{eq:lagrangian_quantile_2}
\mu_{1}+\mu_{2}=\lambda
\end{equation}
We now distinguish the following cases:
\begin{enumerate}
\item $w^{\top}x>y$. This means that $(1-\tau)(w^{\top}x-y)>0$ and therefore
the second constraint suggests $\xi>0$. This implies that $\tau(y-w^{\top}x)-\xi<0$, in other words
the first constraint is not tight. Using the KKT complementary slackness
condition $\mu_{1}(\tau(y-w^{\top}x)-\xi)=0$ we conclude
that in this case $\mu_{1}=0$ and hence $\mu_{2}=\lambda$. Hence
(\ref{eq:lagrangian_quantile}) becomes $w=w_{t}-\lambda(1-\tau)x$. 
Using this we can write the condition
$w^{\top}x>y$ in terms of the a priori known $w_{t}$ as: 
$\frac{w_{t}^{\top}x-y}{(1-\tau)x^{\top}x}>\lambda$.
\item $w^{\top}x<y.$ In a similar manner to the previous case, we can show
that $\mu_{2}=0$ using the complementary slackness condition $\mu_{2}((1-\tau)(w^{\top}x-y)-\xi)=0$.
Therefore $\mu_{1}=\lambda$ and $w=w_{t}+\lambda\tau x$. Again the
condition ($w^\top x > y$) can be written in terms of $w_{t}$
as: $\frac{y-w_{t}^{\top}x}{\tau x^{\top}x}>\lambda$.
\item $w^{\top}x=y$. Plugging in $w$ from (\ref{eq:lagrangian_quantile}) 
and rearranging the terms we have
\[
w_{t}^{\top}x+(\tau(\mu_{1}+\mu_{2})-\mu_{2})x^{\top}x=y
\]
Now we use (\ref{eq:lagrangian_quantile_2}) to get an equation
involving only $\mu_{2}$ which yields
$
\mu_{2}=\frac{w_{t}^{\top}x-y}{x^{\top}x}+\tau\lambda
$
and hence $\mu_{1}=\lambda-\mu_{2}=\frac{y-w_{t}^{\top}x}{x^{\top}x}+(1-\tau)\lambda$.
Plugging these back to the condition on $w$ from the Lagrangian we
get $w=w_{t}-\frac{w_{t}^{\top}x-y}{x^{\top}x}x.$
\end{enumerate}
To sum up, the implicit update for the quantile loss is
\[
w=\begin{cases}
w_{t}-\lambda(1-\tau)x, & \frac{w_{t}^{\top}x-y}{(1-\tau)x^{\top}x}>\lambda\\
w_{t}+\lambda\tau x, & \frac{y-w_{t}^{\top}x}{\tau x^{\top}x}>\lambda\\
w_{t}-\frac{w_{t}^{\top}x-y}{x^{\top}x}x & \mathrm{otherwise}\end{cases}
\]
which is the same as the importance invariant update.

\section{MORE INVARIANT UPDATES}

In this section we just list the updates for 
prediction under the logarithmic loss $\ell(\sigma,y)=y\log\frac{y}{\sigma}+(1-y)\log\frac{1-y}{1-\sigma}$
using the link function $\sigma=\frac{1+\tanh(p)}{2}$ (as usual $p=w^\top x$). As in the case without the link
function, we get an explicit form for $y \in \{0,1\}$ which is 
\[
s(h)=
\begin{cases}
\frac{p + \tanh^{-1}\left(\frac{\omega_0 - 1}{\omega_0 + 1}\right)}{x^\top x} & \mbox{if } y=0 \\
\frac{p - \tanh^{-1}\left(\frac{\omega_1 - 1}{\omega_1 + 1}\right)}{x^\top x} & \mbox{if } y=1  \end{cases} 
\]
where 
\begin{eqnarray*}
\omega_0 & = & W(\exp(e^{2p} +2p+ 4h\eta x^\top x))\\
\omega_1 & = & W(\exp(e^{-2p}-2p+ 4h\eta x^\top x))
\end{eqnarray*}
Recall that $W(z)$ is the Lambert W function.

\section{ACTIVE LEARNING}
The active learning algorithm of \cite{beygelzimer2010agnostic} proceeds 
in $T$ rounds and in each round $t$ it considers example $x_t$. It decides
to query the label of $x_t$ by first computing
a rejection threshold based on $G_t$ the difference in (importance weighted) errors of the ERM hypothesis and
an alternative hypothesis. The latter minimizes the ermirical risk subject to predicting
a different label $y_a$ than that predicted by the ERM hypothesis on the current example. To
get a handle on $G_t$ we estimate $t \cdot G_t$ by the minimum importance weight that
$x_t$ would need to have if $x_t$'s label were $y_a$ so that the update caused the classification 
of the example to become $y_a$.

More specifically, for binary problems with $y\in \{-1,1\}$ 
assume $\hat{y} = \sign(w^\top x)$ and hence the alternative label 
is $y_a = -\hat{y}$. Then we want an importance 
weight $h$ such that $(w - s(h)x)^\top x =0$ where
$s(h)$ is computed using $y_a$ in place of $y$. Since 
$s$ satisfies \eqref{eq:diffeq} , by separating variables, integrating both
sides and making use of the initial condition we get 
Therefore we 
get that 
\[
h = s^{-1}(\frac{w^\top x}{x^\top x})
\]
Notice that $s$ doesn't even need to be in closed form to compute $h$.
The reason is that 
the ODE (\ref{eq:diffeq}) yields a solution directly 
in terms of $s^{-1}$ since if we let $F(s)=\eta\left.\frac{\partial \ell}{\partial p}\right|_{p=w^\top x-s x^\top x}$
then by separation of variables in (\ref{eq:diffeq}) we get:
\begin{equation}\label{eq:vw_gk_estimate}
h = \frac1{\eta_t} \int_0^{\frac{w_t^\top x_t}{||x_t||^2}} \frac{du}{\left.\frac{\partial \ell(p,y_a)}{\partial p}\right|_{p=w_t^\top x_t -u||x_t||^2}}
.
\end{equation}

For implicit updates we do not have a general methodology 
but for each loss we have always been able to find a closed 
form solution for the importance. Below we show how to do 
this in the case of logistic loss. 

For the logistic loss the optimization problem that defines the implicit update is
\[
w_{t+1}=\argmin\frac{1}{2}||w-w_{t}||^{2}+h\eta_t\log(1+\exp(-y_a w^{\top}x_t))
\]
Setting the derivative of the above to zero leads to
\[
w_{t+1}=w_{t}+\frac{h\eta_t y_a}{1+\exp(y_a w_{t+1}^{\top}x_t)} x_t.
\]
The update is implicit 
because $w_{t+1}$ is given in terms of itself.
In any case, to find the relevant importance weight we want 
$w_{t+1}^{\top}x_t=0$ and by taking the inner product of the
above equation with $x_t$ and plugging 
$w_{t+1}^{\top}x_t=0$ in we get
\[
0 = w_{t}^\top x_t +\frac{h\eta_t y_a}{2} x_t^\top x_t
\]
\[
h=-\frac{2 w_{t}^{\top}x_t}{\eta_t y_a x_t^\top x_t}
\]

\section{SCATTERPLOTS}
In Figure~\ref{fig:scatter2} we show a few more scatterplots.
The trend is the same as in Figure~\ref{fig:scatter}.
\begin{figure}[t]
\includegraphics[width=0.475\textwidth]{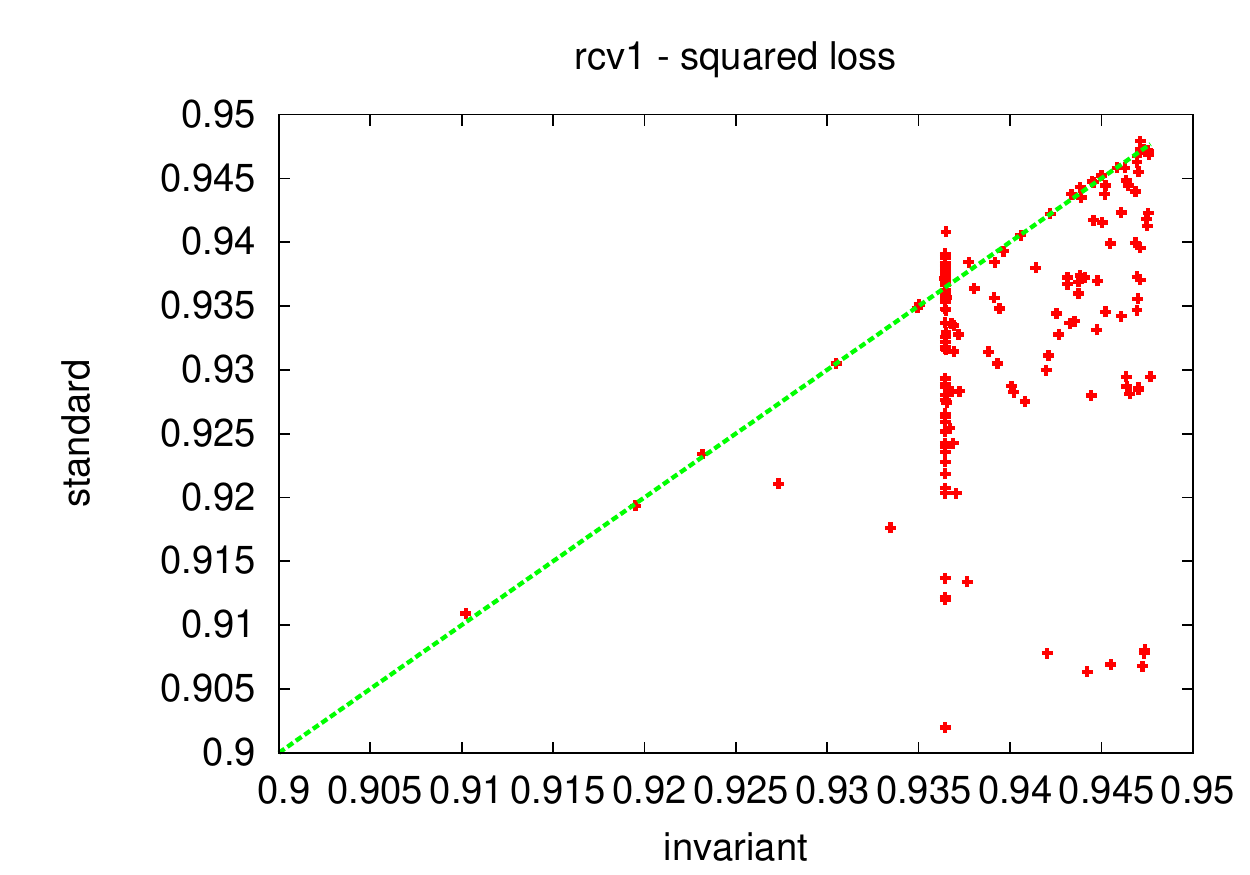}\\
\includegraphics[width=0.475\textwidth]{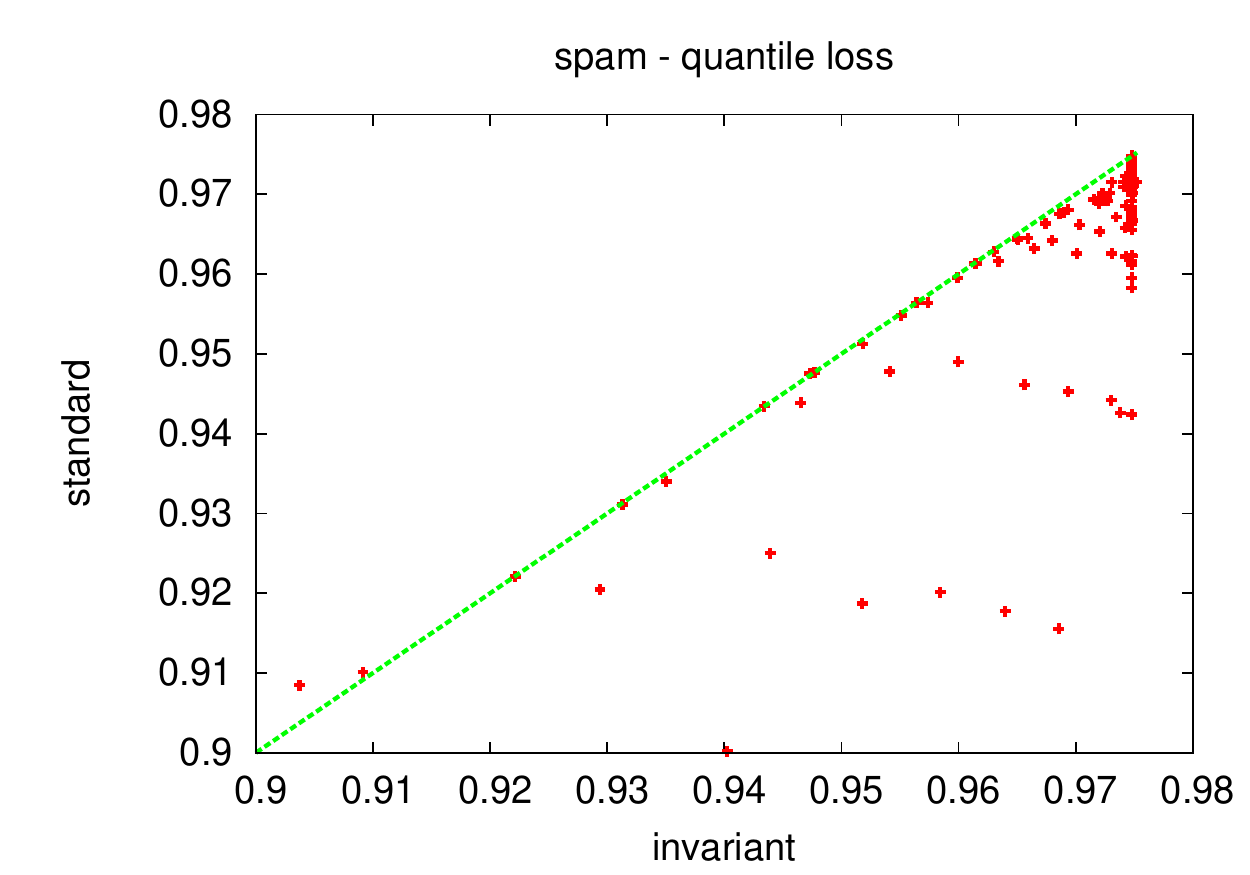}\\
\includegraphics[width=0.475\textwidth]{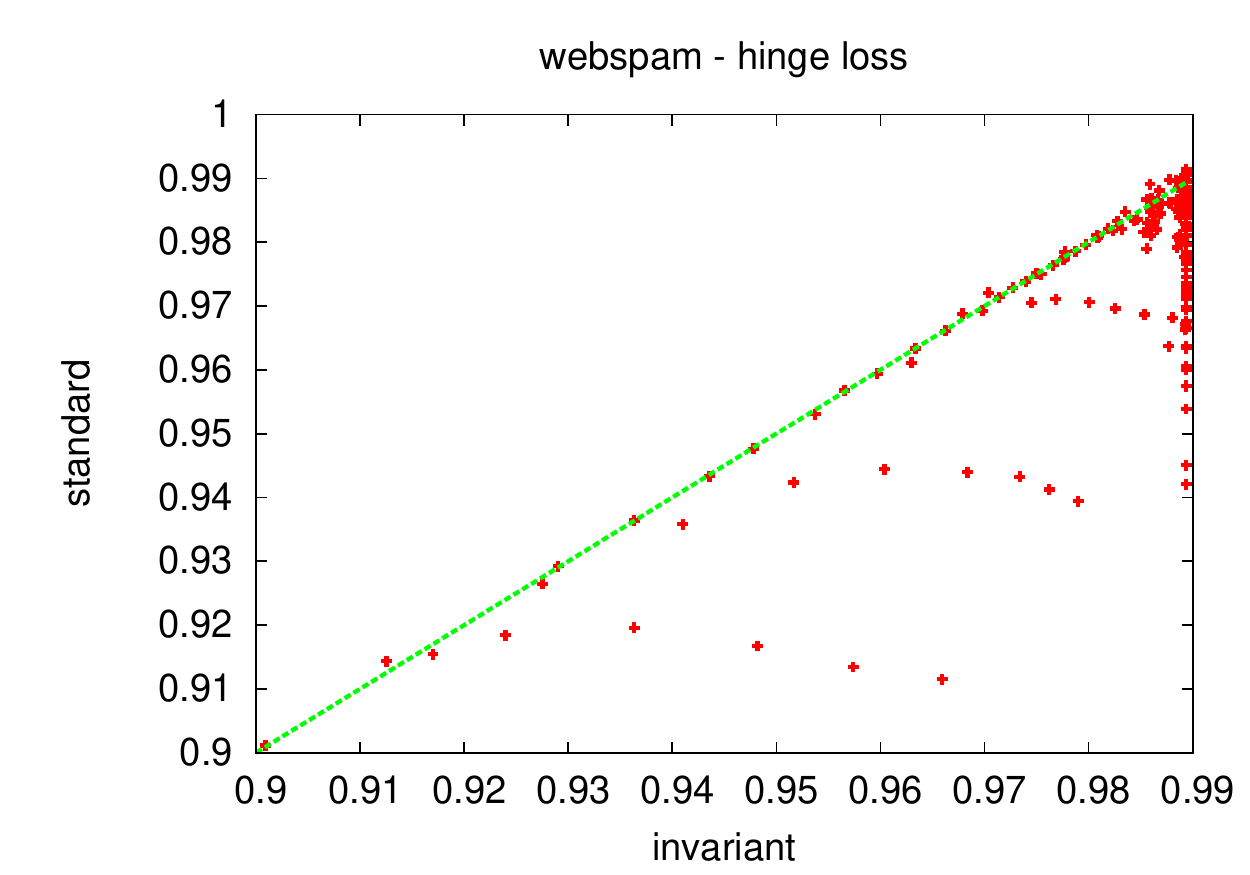}\\
\caption{Scatter plot showing test accuracy with two different updates for various datasets and losses} \label{fig:scatter2}
\end{figure}

\end{document}